\documentclass[preprint,numbers,12pt]{article}
\usepackage{newlfont}
\usepackage{a4wide}
\usepackage{amsmath}
\usepackage{inputenc}
\usepackage{amsfonts}
\usepackage{amssymb}
\usepackage{amsthm}
\usepackage{newlfont}
\usepackage{graphicx}
\usepackage{subfigure}
\usepackage{mathrsfs}
\usepackage[numbers]{natbib}
\usepackage{soul}

\usepackage{color}

\usepackage{comment}

\usepackage[running]{lineno}
\newtheorem{theorem}{Theorem}
\newtheorem{lemma}{Lemma}

\newtheorem{corollary}{Corollary}
\newtheorem{proposition}{Proposition}

\newtheorem{remark}{Remark}

\theoremstyle{definition}

\newcommand{\svm}{\mathrm{(SVM)}}
\newcommand{\nsvm}{\mathrm{(N\mbox{-}SVM)}}
\newcommand{\rsvm}{\mathrm{(R\mbox{-}SVM)}}
\newcommand{\odrsvm}{\mathrm{(OD\mbox{-}SVM)}}
\newcommand{\rsvmc}[1]{\mathrm{(R\mbox{-}SVM(}{#1}\mathrm{))}}
\newcommand{\odrsvmc}[1]{\mathrm{(OD\mbox{-}SVM(}{#1}\mathrm{))}}

\newcommand{\cm}{\mathrm{(CM)}}
\newcommand{\rcm}{\mathrm{(R\mbox{-}CM)}}
\newcommand{\odrcm}{\mathrm{(OD\mbox{-}CM)}}

\newcommand{\rcmd}[1]{\ensuremath\mathrm{(R\mbox{-}CM(}{#1}\mathrm{))}}
\newcommand{\odrcmd}[1]{\ensuremath\mathrm{(OD\mbox{-}CM(}{#1}\mathrm{))}}

\newcommand{\up}{\mathrm{(P)}}
\newcommand{\rp}{\mathrm{(R\mbox{-}P)}}
\newcommand{\dup}{\mathrm{(D\mbox{-}P)}}
\newcommand{\drp}{\mathrm{(DR\mbox{-}P)}}
\newcommand{\odup}{\mathrm{(OD\mbox{-}P)}}
\newcommand{\prj}{\mathrm{P}}

\newcommand{\sth}{\mathrm{s.t.}}
\newcommand{\R}{\mathbb{R}}
\newcommand{\argmax}{\mathrm{argmax}}
\newcommand{\ch}{co}

\newcommand{\sgn}{\mathrm{sign}}
\newcommand{\dist}{\mathrm{dist}}

\newcommand{\cch}{\overline{co}}

\newcommand{\B}{\mathcal{B}}
\newcommand{\X}{X}
\newcommand{\Xs}{X^*}

\newcommand{\lsc}{lower-semicontinuous }
\newcommand{\usc}{upper-semicontinuous }
\newcommand{\xo}[1]{\ensuremath{\overline{#1}}}
\newcommand{\dea}[2]{\ensuremath{{#1}=1,\dots,{#2}}}
\newcommand{\nd}[1]{\ensuremath{\|{#1}\|}} 

\begin{document}

\title{Robust SVM Optimization in Banach spaces}
\author{Mohammed Sbihi$^\dag$ \and Nicolas Couellan$^\ddag$}

\maketitle

\noindent {\small $^\dag$ ENAC, Universit\'e de Toulouse\\
$\qquad$ mohammed.sbihi@enac.fr\\
$^\ddag$ ENAC and  Institut de Math\'{e}matiques de Toulouse, Universit\'{e} de Toulouse\\
$\qquad$ nicolas.couellan@recherche.enac.fr}



\begin{abstract}
We address the issue of binary classification in Banach spaces in presence of uncertainty. We show that a number of results from classical support vector machines theory can be appropriately generalised to their robust counterpart in Banach spaces. These include the Representer Theorem, strong duality for 
the associated Optimization problem as well as their geometric  interpretation. Furthermore, we propose a game theoretic interpretation by expressing a Nash equilibrium problem formulation for the more general problem of finding the closest points in two closed convex sets when the underlying space is reflexive and smooth.

\end{abstract}

\section{Introduction}

Support Vector machines (SVM) \cite{Vap,scholkopf} have been widely used for data classification. Their success is due to sound theoretical foundations and good generalization properties. They address the classification problem by finding the hyperplane that achieves maximum sample margin which leads to minimizing the norm of the classifier parameters.
The standard assumption is that the training data (or their features) lie in a Hilbert space. However, in some applications where objects are complex such as images, signals, trajectories in robotics or aeronautic, data representation may turn out to be restrictive or even inefficient in the Euclidean setting. It might be interesting to consider more general representation spaces that better capture and preserve the topological properties of the training samples \cite{DelEtAl}. For instance, Banach spaces may be used to model images in a very general manner \cite{AdlLun}.  Continuous image models that do not rely on the concept of pixel discretization can be regarded as living in the space of measurable functions over the unit square. The use of a specific norm defines a choice of distances between images that can account for specific image features, like the position of edges with Sobolev norms. 

A few studies have demonstrated that classical binary classification formulation may be derived also in non-Euclidean spaces. For example, in \cite{DerLee} a semi-inner-product is considered to formulate a binary classification problem in Banach spaces. In \cite{LinYe}, the author also proposed a non-Euclidean setting. General kernels methods in Banach spaces were also investigated in \cite{XuYe,ZhaEtAl}.  
 
Data uncertainties are usually not taken into account in classification models, although they occur most of the time. In order to design models that are immune to noise, robust formulations of SVM models have been proposed in the past \cite{BenGhaNem,CouJan,Tra}. The idea is to consider bounded additive noise perturbations of input samples and formulate a robust counterpart training optimization problem when considering worst case scenarios.

Combining the above concepts by considering data and uncertainties that lie in general Banach spaces, the main contribution of this study is to propose a theoretical framework that generalizes from an optimization point of view the concept of robust SVM in such spaces.
To do so, extending results from robust optimization duality  \cite{BecBen}, an optimistic dual counterpart problem is derived and robust strong duality is shown to hold under some linearity properties with respect to the uncertainties. Section \ref{robust-duality} covers these aspects. The application of these duality results to the case of robust SVM  nicely lead to a representer theorem in Banach spaces. Unlike \cite{DerLee} where a supporting semi-inner product is used in the non Euclidean setting in place of the inner-product, we propose to use the duality product and show that the representer theorem still holds in this case. Additionally, an uncertain hard margin separation problem and its robust counterpart in Banach spaces are formulated. Furthermore a geometric perspective of the problem is proposed. It extends the  interpretation of SVM separation in Euclidean spaces given in \cite{BenBre}. This is the aim of sections \ref{representer-theorem} and \ref{geometric-view}.
Finally, from this geometric point of view, a game theoretic formulation  for the robust classification problem is further proposed, extending in Banach spaces preliminary results from \cite{Cou}. These developments are discussed in Section \ref{game-formulation}.

 \section{Problem statement}

Let X be a Banach space and $\Xs$ be its dual, that is, the space of all real continuous linear functionals on $X$. We recall that $\Xs$ is a Banach space endowed with dual norm defined by
$$
\|f\|=\sup_{\|x\| \leq 1}|f(x)|, \quad \forall f \in \Xs.
$$
There is a natural duality between $X$ and $\Xs$ determined by the bilinear functional
$\langle\cdot,\cdot\rangle : X\times \Xs \rightarrow \R$ defined by
$$
\langle x, f\rangle =f (x); \forall x\in X, f\in \Xs.
$$

We first recall the standard SVM methodology \cite{Vap}  to find the maximum margin separating hyperplane between two classes of data points.

Let $x_i\in \X$ be a collection of input training vectors for \dea{i}{m}
and $y_i \in \{-1,1\}$ be their corresponding labels. If the data are linearly separable, then there exists a  linear functional $w\in \Xs$  and an offset $b \in \R$ such 
that $y_i (\langle x_i,w\rangle+b)>0$ for all \dea{i}{m}. By rescaling  $w$ and $b$, we may assume without loss of generality that the points closest to the hyperplane $H(w,b):=\{x\in X| \langle x,w\rangle+b=0\}$ satisfy $|\langle  x_i, w\rangle +b|=1$. Thus $H$ may be placed in the canonical form   $ y_i (\langle x_i ,w \rangle  + b)\geq 1,$  for all $\dea{i}{m}$.  With this form, the margin of the hyperplane is $\|w\|^{-1}$ (see (\ref{distToHyp})).
We obtain the SVM problem
$$
\svm\qquad 
\begin{array}{ll}
 \displaystyle\min_{w \in \Xs} & \frac{1}{2}\nd{w}^2 \\
 \sth  & y_i(\langle x_i, w \rangle + b)\geq 1,\quad \dea{i}{m}.
\end{array} 
$$
The classifier is then  given by $f (x) = \sgn(\langle  x, w \rangle   + b)$. It is worth mentioning that unlike the formulation given in \cite{DerLee}, we use the duality product instead of a semi-inner-product.
Considering now instead a set of noisy training vectors $\{\tilde{x}_i  \in \X, \quad \dea{i}{m} \}$    where $ \tilde{x}_i= x_i+\delta_i$ for all $\dea{i}{m}$ and $\delta_i$ is a random perturbation. This can be captured by the following  (noisy SVM) problem
$$
\nsvm\qquad 
\begin{array}{ll}
 \displaystyle\min_{w \in \Xs} & \frac{1}{2}\nd{w}^2 \\
 \sth  & y_i(\langle x_i+\delta_i, w \rangle + b)\geq 1,\quad \dea{i}{m}.
\end{array} 
$$

Observe that the problem involves the random variable $\delta_i$ and  can not
be solved as such. Extra knowledge on the perturbations is needed to transform it into a
deterministic and numerically solvable problem. In general  the perturbation $\delta_i$ is  known to reside in some uncertainty set $\Delta_i \subset \X$. For instance, \cite{CouJan} considers,  when $\X=\R^n$, the  uncertainty set as  $\|\Sigma^{1/2}\delta_i \|_p\leq \gamma_i,$ \dea{i}{m}, where is $\Sigma$ is some positive definite matrix and $p\geq 1$.  Various choices of $\Sigma_i$  and $p$  will lead to various types of uncertainties such as for example box-shaped uncertainty ($\|\delta_i\|_\infty \leq \gamma_i$), spherical uncertainty  ($\|\delta_i\|_2\leq \gamma_i$), or ellipsoidal uncertainty ($\delta_i^T \Sigma^{-1} \delta_i \leq \gamma^2_i$).
To design a robust model, one has to satisfy the inequality constraint in Problem (N-SVM) for every realizations 
of $\delta_i$. This can be done by ensuring the constraint in the worst case scenario for $\delta_i$, leading to the following robust counterpart optimization problem:
$$
\rsvm\qquad 
\begin{array}{ll}
 \displaystyle\min_{w \in \Xs} & \frac{1}{2}\nd{w}^2 \\
 \sth  & \displaystyle\min_{x_i\in K_i} y_i(\langle x_i, w \rangle + b)\geq 1,\quad \dea{i}{m},
\end{array} 
$$
with $K_i=x_i+\Delta_i$ and with an abuse of notation $\tilde{x}_i$ is denoted by $x_i$.

 As previously announced, the aim of this paper is to generalize some known results \cite{BenBre,Cou,DerLee} from classical SVM to their robust counterpart in Banach spaces. The following section prepares the ground by recalling some facts about robust optimization and by generalizing a robust strong duality result \cite{BecBen} to a Banachic framework  tailored to robust SVM related problems.

\section{A robust Optimization detour}\label{robust-duality}

We state and adapt  in this section some results from robust optimization to our context.  Consider a general uncertain optimization problem on some Banach space $\B$
$$
\up \qquad \inf_{x\in\B}\left\{f(x): g_i(x,u_i)\leq 0,\ \dea{i}{m}\right\},
$$
where $f:  \B \rightarrow \R$ is a \lsc convex function and $g_i :\B \times U_i \rightarrow \R, \dea{i}{m}$, $g_i(\cdot,u_i)$ is convex continuous,  $g_i(x,\cdot)$ is \usc   and $u_i$ is the uncertain parameter which is only known to reside in certain convex compact uncertainty set $U_i.$
Robust optimization, which has emerged as a powerful deterministic approach for studying mathematical programming under uncertainty \cite{BecBen,BenGhaNem,JeyEtAl} associates with the uncertain program $\up$
its robust counterpart,
$$
\rp \qquad \inf_{x\in\B}\left\{f(x): \sup_{u_i\in U_i} g_i(x,u_i)\leq 0,\ \dea{i}{m}\right\},
$$
where the uncertain  constraints are enforced for every possible value of the parameters within their prescribed uncertainty sets $U_i$.
The functions $G_i: \B\ni x\mapsto  \sup_{u_i\in U_i} g_i(x,u_i)$, \dea{i}{m}
 are convex and continuous as point-wise maxima of convex continuous functions. 
It is known \cite[Theorem 3.9]{BarPre} that under the following Slater condition:
\begin{equation}\label{SC}
\mbox{There exists a point } x_0 \in \B \mbox{ such that } G_i(x_0)<0,\  \dea{i}{m},
\end{equation} a  point $\xo{x}\in \B$ is an optimal solution for $\rp$ if and only if there exists $\xo{\lambda} \in \mathbb{R}_+^m$ such that 
\begin{eqnarray}
& &G_i(\xo{x}) \leq 0,\quad  \dea{i}{m},\label{CFeas}\\
& & 0\in \partial f(\xo{x}) +\displaystyle\sum_{i=1}^m \xo{\lambda}_i \partial G_i(\xo{x}),\label{CCrit}\\
& &\xo{\lambda}_i G_i(\xo{x})= 0, \quad\dea{i}{m}.\label{CComp}   
\end{eqnarray}
where for a convex function $h: \B\rightarrow \R$, $\partial h(x)$ denotes the Fenchel subdifferential defined by
$$
\partial h(x)=\{x^*\in\B^*: h(y)\geq h(x) + \langle y-x,x^*\rangle, \quad \forall y\in\B\}.
$$

The following  preparatory result refines  further the conditions (\ref{CFeas})--(\ref{CComp}) under linearity assumption with respect to the uncertainties.
\begin{proposition}\label{prop-CNS} Suppose that $U_i$ is a weakly convex subset of some Banach space $\mathcal{C}$
and $g_i(x,\cdot)\in \mathcal{C}^*$ for all $x\in \B.$ Under assumption (\ref{SC}), a point $\xo{x}\in \B$ is a minimizer to $\rp$ if and only if
there exist $\xo{\lambda}\in \R^m_+$ and $ \xo{u} \in U:=\prod_{i=1}^m U_i$ such that
\begin{eqnarray}
&& \sup_{u_i\in U_i} g_i(\xo{x},u_i)\leq 0, \quad \dea{i}{m},\label{kktRfeas}\\
&& 0\in \partial f(\xo{x}) +\displaystyle\sum_{i=1}^m \xo{\lambda}_i \partial_x g_i(\xo{x},\xo{u}_i), \label{kktRcrit}\\
&& \xo{\lambda}_i g_i(\xo{x},\xo{u}_i)=0\quad \dea{i}{m}.\label{kktEcomp}
\end{eqnarray}
\end{proposition}

\begin{proof}
Let us first show that
\begin{equation}\label{HypSousDif}
\partial G_i(x)=\displaystyle\bigcup_{u_i\in U_i(x)} \partial_x g_i(x,u_i)
\end{equation}
where $U_i(x)=\argmax_{u_i\in U_i} g_i(x,u_i)$.  
It is known \cite{HanEtAl} that
$$
\partial G_i(x)=\cch\left(\displaystyle\bigcup_{u_i\in U_i(x)} \partial_x g_i(x,u_i)\right)
$$
where $\cch$ indicates the closure of the convex hull with respect to weak$*$ topology $\sigma(\B^*,B)$, so proving (\ref{HypSousDif}) amounts to prove that $\displaystyle\cup_{x_i\in U_i(x)} \partial_x g_i(x,u_i)$ is convex and weakly* closed. Observe that $U_i(x)$ is convex and a closed subset of $U_i$, hence by the linearity of $g_i$ with respect to $u_i$, it follows that $\bigcup_{u_i\in U_i(x)} \partial_x g_i(x,u_i)$ is convex. So to prove that its  weak* closedness, it suffices to prove its sequential weak* closedness. To this end, let $s_n\in \partial_x g_i(x,u_i^n)$, with $u_i^n\in U_i(x)$,
converging to some $s\in \B^*$. As $U_i(x)$ is weakly compact, $(s_n)_n$ admits a convergent sub-sequence, still denoted by $(s_n)_n$, converging to some $u_i\in U_i(x)$. By letting $n$ to $+\infty$ in the inequality $ g_i(y,u_i^n)\geq g_i(x,u_i^n) +\langle s_n, y-x\rangle$ we get  $ g_i(y,u_i)\geq g_i(x,u_i) +\langle s, y-x\rangle$, which shows that $s\in \bigcup_{u_i\in U_i(x)} \partial_x g_i(x,u_i)$.

Let us now consider a point $(\xo{x},\xo{\lambda},\xo{u})$ satisfying  (\ref{kktRfeas})--(\ref{kktEcomp}). First note that (\ref{CFeas}) is not else but (\ref{kktRfeas}). For indices $i$ such that  $\xo{\lambda}_i=0$ it is clear that (\ref{CComp}) is satisfied.  If $\xo{\lambda}_i>0$, then by (\ref{kktEcomp}) $g_i(\xo{x},\xo{u}_i)=0$ which combined with (\ref{kktRfeas}) yields $0=g_i(\xo{x},\xo{u}_i)=\sup_{u_i\in U_i} g_i(\xo{x},u_i)=G_i(\xo{x})$, so (\ref{CComp}) is satisfied. Moreover,  by (\ref{HypSousDif}) $ \partial_x g_i(\xo{x},\xo{u}_i)\subset  G_i(\xo{x})$. Summing over $i$ gives $0\in \partial f(\xo{x}) + \sum_{i=1}^m \xo{\lambda}_i \partial_x g_i(\xo{x},\xo{u}_i) \subset \partial f(\xo{x}) + \sum_{i=1}^m \xo{\lambda}_i \partial G_i(\xo{x}).$ Consequently, $(\xo{x},\xo{\lambda})$ satisfy  (\ref{CFeas})--(\ref{CComp}). 
Let now $(\xo{x},\xo{\lambda})$ verifying (\ref{CFeas})--(\ref{CComp}).  For  such $\xo{\lambda}$,
by (\ref{CCrit}) there exists $v\in\partial f(\xo{x})$, $v_i\in \partial G_i(\xo{x})$ such that 
$0=v+\sum_{i=1}^m \lambda_i v_i$. Using (\ref{HypSousDif}), for each $i$, there exists $\xo{u}_i\in U_i(\xo{x})$ such that $v_i\in \partial g_i(\xo{x},\xo{u}_i).$ Finally, we can readily check that the resulting triplet $(\xo{x},\xo{\lambda},\xo{u})$ satisfies  (\ref{kktRfeas})--(\ref{kktEcomp}). $\qed$
\end{proof}

\begin{remark}
Proposition \ref{prop-CNS} is still valid if we replace the linearity assumption of $g_i(x,\cdot)$ with respect to $u_i$ by (\ref{HypSousDif}). 
\end{remark}

The dual of $\rp$ is given by
$$
\sup_{\lambda\in \R^m_+}\inf_{x\in\B}\left\{f(x)+\sum_{i=1}^m \lambda_i G_i(x)\right\}
$$
which by recalling the definition of $G_i$ becomes
$$
\drp \qquad \sup_{\lambda\in \R^m_+}\inf_{x\in\B}\sup_{u_i\in U_i}\left\{f(x)+\sum_{i=1}^m  \lambda_i g_i(x,u_i)\right\}.
$$

On the other hand, the uncertain dual of $\up$ is given by
$$
\dup \qquad \sup_{\lambda\in \R^m_+}\inf_{x\in\B}\left\{f(x)+\sum_{i=1}^m \lambda_i g_i(x,u_i)\right\}.
$$
The optimistic counterpart of $\dup$ is
$$
\odup \qquad \sup_{u\in U, \lambda\in \R^m_+}\inf_{x\in\B}\left\{f(x)+\sum_{i=1}^m  \lambda_i g_i(x,u_i)\right\}.
$$

By construction, $\inf\rp\geq \sup\drp\geq  \sup\odup$. The authors in \cite{BecBen} have established, in the case of $\B = \R^n$, that robust strong duality (i.e.  $\inf\rp = \max\odup$)   holds between the problems  under the Slater condition whenever each $g_i(x,\cdot), \dea{i}{m}$ is a concave function with respect to  $u_i$.  In other words, optimizing under the worst case scenario in the primal is the same as optimizing under the best case scenario in the dual (''primal worst equals dual best'').  We will establish an analogue result in Banach spaces under some linearity properties with respect to the uncertainties.

By noticing that $\rp$ is equivalent to 
$$
 \inf_{x\in\B}\sup_{u\in U, \lambda\in \R^m_+}\left\{f(x)+\sum_{i=1}^m  \lambda_i g_i(x,u_i)\right\},
$$
$\rp$ and $\odup$ can be viewed as dual to each other with $u$ playing the role of an abstract Lagrange multiplier \cite[p. 460]{Zeid}. So establishing the strong duality amounts to search the existence of a saddle point of the uncertain lagrangian 
$$
L: \B\times ( \R^m_+\times U)\ni (x;\lambda,u) \mapsto f(x)+\sum_{i=1}^m  \lambda_i g_i(x,u_i)
$$
with respect to $\B\times ( \R^m_+\times U)$, that is a point $(\xo{x};\xo{\lambda},\xo{u})\in \B\times ( \R^m_+\times U)$ such that:
\begin{equation}\label{SP}
L(\xo{x}; \lambda,u)\leq L(\xo{x};\xo{\lambda},\xo{u})\leq L(x;\xo{\lambda},\xo{u}), \qquad \forall (x;\lambda,u)\in \B\times ( \R^m_+\times U).
\end{equation}

In the sequel   we say that $(\xo{x};\xo{\lambda},\xo{u})$ is a solution for the  robust primal-optimistic dual $\rp-\odup$ pair if $\xo{x}$ is a solution for $\rp$ and  $ (\xo{\lambda},\xo{u})$ is a solution for $\odup$. By \cite[Theorem 49.B]{Zeid}   $(\xo{x};\xo{\lambda},\xo{u})$ is a solution for the robust primal-optimistic dual pair  $\rp-\odup$ if and only if  $(\xo{x};\xo{\lambda},\xo{u})$ satisfy (\ref{SP}) and in that case  the robust strong duality holds, that is $\min\rp = \max\odup$.

In the following proposition we link the existence of a saddle point to optimality KKT-like conditions (\ref{kktRfeas})--(\ref{kktEcomp}).
\begin{proposition}\label{prop-SP}
Under the assumptions of Proposition \ref{prop-CNS}, a point $(\xo{x};\xo{\lambda},\xo{u})$  is a saddle point of $L$ with respect to $\B\times (\R^m_+\times U)$ if and only if it satisfies (\ref{kktRfeas})--(\ref{kktEcomp}). 
\end{proposition}

\begin{proof}
Suppose that $(\xo{x},\xo{\lambda},\xo{u})$ satisfy (\ref{SP}) then by \cite[Theorem 49.B]{Zeid}  $\xo{x}$ is a solution $\rp$  and $(\xo{\lambda},\xo{u})$ is a solution to $\odup$ so (\ref{kktRfeas}) is satisfied. 
From the right hand of (\ref{SP}), $\xo{x}$ is a minimizer of $L(\cdot,\xo{\lambda},\xo{u})$ and consequently $0\in \partial_x L(\xo{x},\xo{\lambda},\xo{u})=\partial f(\xo{x})+\sum_{i=1}^m \xo{\lambda}_i \partial_x  g_i(\xo{x},\xo{u}_i)$, that is (\ref{kktRcrit}) is satisfied. It remains to show (\ref{kktEcomp}). Consider the no trivial case where  $\xo{\lambda}_j>0$. Again from the left-hand side of (\ref{SP}) and by choosing $\lambda$ such that $\lambda_j=\frac{\xo{\lambda}_j}{2}$ and zero otherwise, we get $\frac{\xo{\lambda}_j}{2} g_j(\xo{x},\xo{u}_j)\geq 0$ which combined with $ \sup_{u_j\in U} g_j(\xo{x},u_j)\leq 0$ leads to $g_i(\xo{x},\xo{u}_j)=0.$ Hence (\ref{kktEcomp}) is satisfied. Consider now a point  $(\xo{x},\xo{\lambda},\xo{u})$ satisfying  (\ref{kktRfeas})--(\ref{kktEcomp}). From $(\ref{kktRcrit})$ there exist $d\in\partial f(\xo{x}),$ $d_i\in \partial_x g_i(\xo{x},\xo{u}_i)$ for \dea{i}{m}   such that
$$
d+\sum_{i=1}^m \xo{\lambda}_i d_i=0.
$$
Moreover, we have for all $x\in X$
\begin{eqnarray*}
f(x)-f(\xo{x})&\geq& \langle   x-\xo{x},d\rangle,\\
g(x,\xo{u}_i)-g(\xo{x},\xo{u}_i)&\geq& \langle   x-\xo{x},d_i\rangle,  \quad\dea{i}{m}.
\end{eqnarray*}
Multiplying by $\xo{\lambda}_i$ ($\geq0$) appropriately and summing up all these  inequalities,
\begin{eqnarray*}
 f(x)+\sum_{i=1}^m \xo{\lambda}_i g(x,\xo{u}_i) - (f(\xo{x})+\sum_{i=1}^m \xo{\lambda}_i 
g(\xo{x},\xo{u}_i))\geq \langle   x-\xo{x}, d+\sum_{i=1}^m \xo{\lambda}_i d_i\rangle= 0
\end{eqnarray*}
that is, $L(\xo{x};\xo{\lambda},\xo{u})\leq L(x;\xo{\lambda},\xo{u})$ for all $x\in \B$.
On the other hand, for any $u\in U$ and $\lambda\in \R^m_+$.
By (\ref{kktRfeas}) then (\ref{kktEcomp})
\begin{eqnarray*}
f(\xo{x})+\sum_{i=1}^m \lambda_ig_i(\xo{x},u_j) \leq f(\xo{x}) =f(\xo{x})+  \sum_{i=1}^m \xo{\lambda}_i g_i(\xo{x},\xo{u}_j)
\end{eqnarray*}
that is, 
$L(\xo{x};\xo{\lambda},u)\leq L(\xo{x};\xo{\lambda},\xo{u})$ for all $(\lambda,u)\in \R^m_+\times U$.
 $\qed$
\end{proof}

\section{An uncertain Representer theorem}\label{representer-theorem}

In the absence of uncertainty \emph{i.e.} $K_i$ reduced to a single element $x_i$ \dea{i}{m},  it is shown in \cite{DerLee} that, although  posed in an infinite-dimensional space, the optimal solution depends only on the metric relations between the data points. In other words, the problem should not depend on the ambient space in which the data are embedded (such property is known as the Representer Theorem in Hilbert space setting).  We shall prove a similar uncertain Representer Theorem in the sense that the optimal solution depends on some realisation of the uncertainty.  

We introduce the following notation
\begin{itemize}
\item $I_+ =\{i : y_i=+1\}$, $I_-= \{i  : y_i=-1\}$ and $I=I_-\cup I_+$,
\item $K=\prod_{i=1}^m K_i,$
\item $K_\circ= \ch \left(\displaystyle\bigcup_{i \in I_\circ} K_i\right)$, $\circ\in\{+,-\} $ ($\ch(A)$ refers to convex hull of $A$)
\end{itemize}
Note that 
$
K_\circ$ is given by 
$$K_\circ=
\left\{\displaystyle\sum_{i\in {I_\circ}}\alpha_i x_i : x_i \in K_i, 0\leq \alpha_i\leq 1, i \in I_\circ \mbox{ and }   \displaystyle\sum_{i \in I_\circ}\alpha_i =1 \right\}
$$
and it is weakly compact as soon as each $K_i$ is weakly compact \cite[Lemma 5.14]{AliBor}.
Moreover we need  to introduce the (normalized) duality mapping \cite[Definition 4.1]{Cio}   $M : X \rightarrow 2^{\Xs}$  defined by
 \begin{equation}\label{eq:dualitymap}
 M(x)= \{x^*\in \X^*; \langle x, x^*\rangle=\|x\|^2=\|x^* \|^ 2\}.
 \end{equation}
The duality mapping  serves as a replacement for the isomorphism $H$ to $H^*$ in the Hilbert case.   It is worth noting that $\partial (\frac{1}{2}\|\cdot\|^2)(x)=M(x).$ 
 
We are now able to state the following uncertain Representer Theorem for $\rsvm$.
\begin{theorem}\label{thm-representer}
If the uncertainty sets $K_i$,  \dea{i}{m}, are convex and weakly compact and 
 \begin{equation}\label{cch}
 K_+\cap K_-=\emptyset,
 \end{equation}
  then $\rsvm$ admits at least one solution. 
If $\Xs$ is strictly convex then $\xo{w}$ is unique. Moreover, 
 $(\xo{w},\xo{b})$  is a minimizer to $\rsvm$ if and only if there exist $\xo{\lambda}\in \R^m_+$ and $
 \xo{x} \in K$ such that
\begin{eqnarray}
&& \max_{x_i\in K}  (1-y_i(\langle x_i, \xo{w} \rangle  + \xo{b}))\leq 0\quad \dea{i}{m},\label{cfeas}\\
&& \xo{w} \in M(\sum_{i=1}^m y_i \xo{\lambda}_i \xo{x}_i),\label{critW}\\
&& \sum_{i=1}^n y_i \xo{\lambda}_i=0,\label{critB}\\
&& \xo{\lambda}_i(1-y_i(\langle \xo{x}_i, \xo{w} \rangle  + \xo{b}))=0\quad \dea{i}{m}.\label{compS}
\end{eqnarray}
\end{theorem}

Note that $\Xs$ is strictly convex iff for all $w_1, w_2\in \Xs$, $w_1\neq w_2$, $\|w_1\|= \|w_2\|=1$, one has $\|\lambda w_1+(1-\lambda)w_2\|<1$, $\forall \lambda \in]0,1[$.
In terms of supporting hyperplanes, this property may be expressed as: \textit{distinct boundary points of the closed unit ball have distinct supporting hyperplanes}. This property is equivalent to say that $X$ is smooth \cite[Theorem 1.101]{BarPre}, that is if for every $x\neq 0$ there exists a unique $x^*$ such that $\|x^*\|=1$ and $\langle x,x^*\rangle  = \|x\|$ or in other words there is exactly one supporting hyperplane through
each boundary point of the closed unit ball.

\begin{proof}
Given another Banach space $F$, we will use in the proof the fact that $(X \times F )^*$ is homeomorphic to $\Xs\times F^*$ via the linear application 
$l : \Xs\times F^* \ni (h,k) \mapsto h \times k \in (X \times  F )^*$, where $(h\times k)(x, y) = h(x) +k(y)$ and $X\times F$ is endowed with the product norm $\|(x,y)\|= \|x\|_X+\|y\|_F.$
Let us define the functions $f$, $g_i$ and $G_i$ by $f(w,b)=\frac{1}{2}\|w\|^2$ and $g_i(w,b,x_i)=1-y_i(\langle x,w \rangle_{\Xs}  + b)$ and
$G(w,b)=\sup_{x_i\in K_i}g_i(w,b,x_i)$.
Let us first show that (\ref{cch}) implies that the Slater condition is satisfied and at the same time the feasible set is not empty. Since $K_-$ and $K_+$ are weakly compact and disjoint, by  \cite{Klee} they can be strictly separated. More precisely there exists $w \in \Xs$ such that  $\inf_{x \in K_+} \langle x, w \rangle > \sup_{x \in K_-} \langle x, w\rangle.$ Let $\alpha,\beta$ such that
$\inf_{x \in K_+} \langle x, w \rangle >\alpha>\beta > \sup_{x \in K_-} \langle x, w \rangle$. Set $w_0= \frac{2w}{\alpha-\beta}$ and $b_0=-\frac{\alpha+\beta}{\alpha-\beta}$ then $g_i(w_0,b_0)<0$, \dea{i}{m}.
The objective function $f$ is weakly* \lsc and coercive on $\Xs$. The feasible set 
$\cap_{i\in I} \cap_{x_i\in K_i}\{(w,b)\in \Xs \times \R : g_i(w,b,x_i)\leq 0\}$
is weakly* closed  because $g(\cdot,\cdot, x_i)$ is weakly* continuous. This guarantees the existence of a solution \cite[Corollary 1.8.]{Cio}.  The feasible set is convex and  the objective function ($\frac{1}{2}\|\cdot\|^2)$ is convex so the set of the minimizers is convex. Given two minimisers $w_1$ and $w_2$, we have $\nd{w_1}=\nd{w_2}=\nd{\frac{w_1+w_2}{2}}$. When $\Xs$ is strictly convex this is possible only  if $w_1=w_2$, which ensures the uniqueness of $w$.

We have 
$\partial_{w,b} g_i(w,b,x_i) =  \{(-y_i x_i,-y_i) \}$ and $\partial f(w,b) =\{(s,0): s \in M(w)\}.$ Applying Proposition \ref{prop-CNS} and remarking that 
$\sum_{i=1}^m  y_i \xo{\lambda}_i \xo{x}_i\in M(\xo{w})$ is equivalent to $ \xo{w} \in M(\sum_{i=1}^m  y_i \xo{\lambda}_i \xo{x}_i)$
ends the proof. $\qed$
\end{proof}
In the $L^p$ case, $1<p<+\infty,$ the duality mapping is single-valued \cite[Corollary 4.10]{Cio} and we obtain the following corollary as in \cite{DerLee}.

\begin{corollary}
In the particular case of $\X=L^p(\Omega)$, $1<p<+\infty,$ the $\rsvm$ separating hyperplane admits the expansion 
$\xo{w} =\frac{\left|\sum_{i=1}^p \xo{\lambda}_i \xo{x}_i\right|^{p-1} {\sgn(\sum_{i=1}^p \xo{\lambda}_i \xo{x}_i )}}{\| \sum_{i=1}^p \xo{\lambda}_i \xo{x}_i \|^{p-2}} \in L^q(\Omega)$ (with $\frac{1}{p}+\frac{1}{q}=1$ and  equality in $L^q$ sense).
\end{corollary}

\section{Duality and geometry in Robust SVM classifiers}\label{geometric-view}

It is shown in \cite{BenBre} that (SVM) is equivalent to following C-Margin formulation:
$$
\cm \qquad 
\begin{array}{ll}
 \displaystyle\min_{w \in \Xs} & \frac{1}{2}\nd{w}^2  -(\alpha-\beta)\\
 \sth  &  \langle x_i, w \rangle \geq \alpha,\quad  i \in I_+,\\
      &  \langle x_i, w \rangle \leq \beta,\quad  i \in I_-.\\
\end{array} 
$$
whose dual is the problem of finding the closest points in  the convex hull of each class.
To $\cm$ we associate the robust version
$$
\rcm \qquad 
\begin{array}{ll}
 \displaystyle\min_{w \in \Xs} & \frac{1}{2}\nd{w}^2  -(\alpha-\beta)\\
 \sth  & \displaystyle\min_{x_i\in K_i} \langle x_i, w \rangle \geq \alpha,\quad  i \in I_+,\\
      & \displaystyle\max_{x_i\in K_i} \langle x_i, w \rangle \leq \beta,\quad  i \in I_-.\\
\end{array} 
$$

Like $\rsvm$, the problem $\rcm$ admits the following Representer Theorem,  the proof of which is identical to Theorem \ref{thm-representer} and therefore omitted.
\begin{theorem}
If the uncertainty sets $K_i$,  \dea{i}{m}, are convex and weakly compact and 
$
 K_+\cap K_-=\emptyset,
$
  then  $\rcm$ admits at least one solution. Moreover, 
 $(\xo{w},\xo{\alpha},\xo{\beta})$  is a minimizer to $\rcm$ if and only if there exists $\xo{\lambda}\in \R^m_+$ and $
 \xo{x} \in K$ such that
\begin{eqnarray*}
&& \max_{x_i\in K_i}(\xo{\alpha}- \langle x_i, \xo{w} \rangle)\leq 0,\quad i \in I_+,\\
&& \max_{x_i\in K_i} ( \langle x_i, \xo{w} \rangle - \xo{\beta})\leq 0,\quad i \in I_-,\\
&& \xo{w} \in M(\sum_{i=1}^m  y_i\xo{\lambda}_i  \xo{x}_i),\\
&& \sum_{i\in I_\circ} \xo{\lambda}_i=1,\quad \circ \in \{+,-\},\\
&& \xo{\lambda}_i(\xo{\alpha}- \langle \xo{x}_i, \xo{w} \rangle)=0,\quad i \in I_+,\\
&& \xo{\lambda}_i(\xo{\beta}- \langle \xo{x}_i, \xo{w} \rangle)=0,\quad i \in I_-.
\end{eqnarray*}
\end{theorem}
We now extend the duality relationship between  $\svm$ and $\cm$ and the resulting geometric interpretation  \cite{BenBre} to their robust counterparts.
Define the uncertain Lagrangian $L_1$ on $\Xs \times \R\times \R^m_+\times K$ by
\begin{eqnarray*}
 L_1(w,b;\lambda,x)&=&  \frac{1}{2}\|w\|^2 +\sum_{i=1}^m
     \lambda_i(1-y_i(\langle x_i, w \rangle  + b)) \\
 &=& \frac{1}{2}\|w\|^2 - \left\langle \sum_{i=1}^m \lambda_i y_i x_i, w \right \rangle + \sum_{i=1}^m \lambda_i  - b\sum_{i=1}^m \lambda_i y_i.\label{noisyLagr}
\end{eqnarray*}
Let us now take a closer look at $\odrsvm$
\begin{equation}\label{primal}
\sup_{(\lambda,x)\in \R^m_+\times K}\inf_{(w,b)\in \Xs\times\R}\ L_1(w,b;\lambda,x).
\end{equation}
Taking the subdifferential of $L_1$ with respect to $w$ and $b$ yields
 \begin{eqnarray}
\partial_w L_1(w,b,\lambda,x) &=& M(w)+ \{-\sum_{i=1}^m \lambda_i y_i x_i\},\label{sdLw}\\
\partial_b L_1(w,b,\lambda,x) &=& -\sum_{i=1}^m \lambda_i y_i. \label{sdLb}
\end{eqnarray}
Letting $0$ belong to the subdifferentials (\ref{sdLw}) and (\ref{sdLb}) to zero gives
 \begin{eqnarray}
 \sum_{i=1}^m \lambda_i y_i x_i &\in &  M(w), \label{sdLw1}\\
 \sum_{i=1}^m \lambda_i y_i &=&0.\label{sdLb1}
\end{eqnarray}
Substituting (\ref{sdLw1}) and (\ref{sdLb1}) in (\ref{noisyLagr}) subject to the relevant constraints yields the dual $\odrsvm$  stated as follows
$$
\odrsvm\qquad 
\begin{array}{rl}
 \displaystyle\sup_{(\lambda,x)\in \R_+^m\times K} &\sum_{i=1}^m \lambda_i - \frac{1}{2}\left\| \sum_{i=1}^m \lambda_i y_i x_i \right\|^2 \\
 \sth  & \sum_{i=1}^m \lambda_i y_i =0.
\end{array} 
$$
By the same arguments and considering  $L_2$ on $\Xs \times \R\times\R\times \R^m_+\times K$ defined by
$$
 L_2(w,\alpha,\beta;\lambda,x)=\frac{1}{2}\|w\|^2-(\alpha-\beta) +\sum_{i\in I_+} 
     \lambda_i(\alpha-\langle x_i, w \rangle)-\sum_{i\in I_-} 
     \lambda_i(\beta-\langle x_i, w \rangle),\\
$$
we get
$$
\odrcm\qquad 
\begin{array}{rl}
 \displaystyle\sup_{(\lambda,x)\in \R_+^m\times K} &  - \frac{1}{2}\left\| \sum_{i=1}^m \lambda_i y_i x_i \right\|^2 \\
 \sth  & \sum_{i\in I_\circ} \lambda_i =1, \quad \circ \in \{+,-\}
\end{array} 
$$
which is not else but the problem of minimizing the (squared) distance between the two convex hulls $K_+$ and $K_-$.
%

The following theorem states that $\rsvm$ and $\rcm$ are equivalent.

\begin{theorem}
Assume that the uncertainty sets $K_i$,  \dea{i}{m}, are convex and weakly compact and $K_+\cap K_-=\emptyset$, then
\begin{enumerate}
\item 
If $(\xo{w},\xo{b};\xo{\lambda},\xo{x})$ is a solution to the pair $\rsvm-\odrsvm$, then 
$$\left(\frac{2\xo{w}}{\sum_{i=1}^{m}\xo{\lambda}_i} , \frac{2(1-\xo{b})}{\sum_{i=1}^{m}\xo{\lambda}_i},\frac{2(-1-\xo{b})}{\sum_{i=1}^{m}\xo{\lambda}_i};
\frac{2\xo{\lambda}}{\sum_{i=1}^{m}\xo{\lambda}_i},\xo{x}\right)$$
  is a solution to the pair $\rcm-\odrcm$.
\item
If $(\xo{w},\xo{\alpha},\xo{\beta};\xo{\lambda},\xo{x})$    is a solution to the pair $\rcm-\odrcm$ then 
$$\left(\frac{2\xo{w}}{\xo{\alpha}-\xo{\beta}} ,
-\frac{\xo{\alpha}+\xo{\beta}}{\xo{\alpha}-\xo{\beta}};\frac{2\xo{\lambda}}{\xo{\alpha}-\xo{\beta}},\xo{x}\right)$$
  is a solution to the pair $\rsvm-\odrsvm$.
\end{enumerate}
\end{theorem}
\begin{proof}
We use Proposition \ref{prop-SP} to link the saddles points of $L_1$ and $L_2$. Let us first observe that under the assumptions of the theorem, $\sum_{i=1}^m \xo{\lambda}_i\neq 0$ and $\xo{\alpha}-\xo{\beta}\neq 0.$ Suppose that $(\xo{w},\xo{\alpha},\xo{\beta};\xo{\lambda},\xo{x})$ is a saddle point of $L_2$, then for all $ (w,\alpha ,\beta;\lambda,x)\in \X\times \R\times\R\times \R^m_+\times K$  we have
     
\begin{eqnarray*}
&&\frac{1}{2}\|\xo{w}\|^2-(\xo{\alpha}-\xo{\beta}) +\sum_{i\in I_+} 
     \lambda_i(\xo{\alpha}-\langle x_i, \xo{w} \rangle)-\sum_{i\in I_-} 
     \lambda_i(\xo{\beta}-\langle x_i, \xo{w} \rangle)\\
&\leq& \frac{1}{2}\|\xo{w}\|^2-(\xo{\alpha}-\xo{\beta}) +\sum_{i\in I_+} 
    \xo{\lambda}_i(\xo{\alpha}-\langle \xo{x}_i, \xo{w} \rangle)-\sum_{i\in I_-} 
    \xo{\lambda}_i(\xo{\beta}-\langle \xo{x}_i, \xo{w} \rangle)\\
&\leq &  \frac{1}{2}\|w\|^2-(\alpha-\beta) +\sum_{i\in I_+} 
    \xo{\lambda}_i(\alpha-\langle \xo{x}_i, w \rangle)-\sum_{i\in I_-} 
    \xo{\lambda}_i(\beta-\langle \xo{x}_i, w \rangle).\\
\end{eqnarray*}     
In particular, by choosing $(\frac{\xo{\alpha}-\xo{\beta}}{2})w$, $\frac{\xo{\alpha}-\xo{\beta}}{2}(1-b)$,$-\frac{\xo{\alpha}-\xo{\beta}}{2}(1+b)$ and $\frac{\xo{\alpha}-\xo{\beta}}{2}\lambda$ instead of $w$, $\alpha$, $\beta$ and $\lambda$ respectively we get that for all
 $ (w,b;\lambda,x)\in \X\times \R\times \R^m_+\times K$ 
\begin{eqnarray*}
&&\frac{1}{2}\|\xo{w}\|^2-(\xo{\alpha}-\xo{\beta}) +\sum_{i\in I_+} 
     \lambda_i\left(\frac{\xo{\alpha}-\xo{\beta}}{2}\right)(\xo{\alpha}-\langle x_i, \xo{w} \rangle)\\
      && \hspace{4cm}  -\sum_{i\in I_-} 
     \lambda_i\left(\frac{\xo{\alpha}-\xo{\beta}}{2}\right)(\xo{\beta}-\langle x_i, \xo{w} \rangle)\\
&\leq& \frac{1}{2}\|\xo{w}\|^2-(\xo{\alpha}-\xo{\beta}) +\sum_{i\in I_+} 
    \xo{\lambda}_i(\xo{\alpha}-\langle \xo{x}_i, \xo{w} \rangle)-\sum_{i\in I_-} 
    \xo{\lambda}_i(\xo{\beta}-\langle \xo{x}_i, \xo{w} \rangle)\\
&\leq &  \frac{1}{2}\left\|\left(\frac{\xo{\alpha}-\xo{\beta}}{2}\right)w\right\|^2 -  \left(\xo{\alpha}-\xo{\beta}\right) \\
    && \hspace{0.5cm} +\sum_{i\in I_+} 
    \xo{\lambda}_i \left(\frac{\xo{\alpha}-\xo{\beta}}{2}\right) (1-b-\langle \xo{x}_i, w \rangle)+\sum_{i\in I_-} 
    \xo{\lambda}_i \left(\frac{\xo{\alpha}-\xo{\beta}}{2}\right) (1+b+\langle \xo{x}_i, w\rangle).\\
\end{eqnarray*}   
Dividing by $(\frac{\xo{\alpha}-\xo{\beta}}{2})^2$ yields
\begin{eqnarray*}
&&\frac{1}{2}\left\| \left(\frac{2\xo{w}}{\xo{\alpha}-\xo{\beta}}\right) \right\|^2+
\sum_{i\in I_+} 
     \lambda_i \left( \frac{ 2\xo{\alpha}}{\xo{\alpha}-\xo{\beta}} -\left\langle x_i, \frac{2\xo{w}}{\xo{\alpha}-\xo{\beta}}\right\rangle\right) \\
  && \hspace{4cm}   -\sum_{i\in I_-} 
     \lambda_i \left( \frac{ 2\xo{\beta}}{\xo{\alpha}-\xo{\beta}} -\left\langle x_i, \frac{2\xo{w}}{\xo{\alpha}-\xo{\beta}}\right\rangle\right)\\
&\leq &\frac{1}{2}\left\| \left(\frac{2\xo{w}}{\xo{\alpha}-\xo{\beta}}\right) \right\|^2+
\sum_{i\in I_+} 
    \frac{ 2\xo{\lambda}_i}{\xo{\alpha}-\xo{\beta}} \left( \frac{ 2\xo{\alpha}}{\xo{\alpha}-\xo{\beta}} -\left\langle \xo{x}_i, \frac{2\xo{w}}{\xo{\alpha}-\xo{\beta}}\right\rangle\right)\\
 && \hspace{4cm} -\sum_{i\in I_-} 
     \frac{ 2\xo{\lambda}_i}{\xo{\alpha}-\xo{\beta}} \left( \frac{ 2\xo{\beta}}{\xo{\alpha}-\xo{\beta}}-\left\langle \xo{x}_i, \frac{2\xo{w}}{\xo{\alpha}-\xo{\beta}}\right\rangle\right)\\
&\leq &  \frac{1}{2}\left\|w\right\|^2  +\sum_{i\in I_+} 
   \frac{ 2\xo{\lambda}_i}{\xo{\alpha}-\xo{\beta}}  \left(1-b-\langle \xo{x}_i, w \rangle\right)+\sum_{i\in I_-} 
 \frac{ 2\xo{\lambda}_i}{\xo{\alpha}-\xo{\beta}}\left(1+b+\langle \xo{x}_i, w\rangle\right).
\end{eqnarray*}  
By remarking that $ \frac{ 2\xo{\alpha}}{\xo{\alpha}-\xo{\beta}}=
1+ \frac{\xo{\alpha}+\xo{\beta}}{\xo{\alpha}-\xo{\beta}}$ and $ \frac{ 2\xo{\beta}}{\xo{\alpha}-\xo{\beta}}=
-1+ \frac{\xo{\alpha}+\xo{\beta}}{\xo{\alpha}-\xo{\beta}}$ it follows 
 \begin{eqnarray*}
L_1\left(\frac{2\xo{w}}{\xo{\alpha}-\xo{\beta}} ,
-\frac{\xo{\alpha}+\xo{\beta}}{\xo{\alpha}-\xo{\beta}};\lambda ,x\right) &\leq& L_1\left(\frac{2\xo{w}}{\xo{\alpha}-\xo{\beta}} ,
-\frac{\xo{\alpha}+\xo{\beta}}{\xo{\alpha}-\xo{\beta}}; \frac{2\xo{\lambda}}{\xo{\alpha}-\xo{\beta}},\xo{x}\right)
\leq L_1\left(w,b; \frac{2\xo{\lambda}}{\xo{\alpha}-\xo{\beta}},\xo{x}\right)
\end{eqnarray*} 
 which means that the point $\left(\frac{2\xo{w}}{\xo{\alpha}-\xo{\beta}} ,
-\frac{\xo{\alpha}+\xo{\beta}}{\xo{\alpha}-\xo{\beta}};\frac{2\xo{\lambda}}{\xo{\alpha}-\xo{\beta}},\xo{x}\right)$  is a saddle point of $L_1$.
Conversely, consider a saddle point $(\xo{w},\xo{b};\xo{\lambda},\xo{x})$ of $L_1$, that  is for all $ (w,b;\lambda,x)\in \X\times\R\times \R^m_+\times K$ 
\begin{eqnarray*}
L_1(\xo{w},\xo{b};\lambda,x)\leq L_1(\xo{w},\xo{b};\xo{\lambda},\xo{x})\leq  L_1(w,b;\xo{\lambda},\xo{x}).
\end{eqnarray*} 
Like before by choosing $\frac{\sum_{i=1}^{m}\xo{\lambda}_i}{2}w$,  $\frac{\sum_{i=1}^{m}\xo{\lambda}_i}{2} (\alpha-\beta) $ and $\frac{\sum_{i=1}^{m}\xo{\lambda}_i}{2}\lambda$ instead of $w$, $b$ and $\lambda$ and then dividing by $(\frac{\sum_{i=1}^{m}\xo{\lambda}_i}{2})^2$ we obtain 
 \begin{eqnarray*}
&& L_2\left(\frac{2\xo{w}}{\sum_{i=1}^{m}\xo{\lambda}_i} , 
\frac{2(1-\xo{b})}{\sum_{i=1}^{m}\xo{\lambda}_i},\frac{2(-1-\xo{b})}{\sum_{i=1}^{m}\xo{\lambda}_i};\lambda ,x\right)\\
&\leq & L_2\left(\frac{2\xo{w}}{\sum_{i=1}^{m}\xo{\lambda}_i} , 
\frac{2(1-\xo{b})}{\sum_{i=1}^{m}\xo{\lambda}_i},\frac{2(-1-\xo{b})}{\sum_{i=1}^{m}\xo{\lambda}_i};\frac{2\xo{\lambda}}{\sum_{i=1}^{m}\xo{\lambda}_i},\xo{x}\right)\\
&\leq& L_2\left(w,\alpha,\beta;
\frac{2\xo{\lambda}}{\sum_{i=1}^{m}\xo{\lambda}_i},\xo{x}
\right)
\end{eqnarray*} 
 which means that the point $\left(\frac{2\xo{w}}{\sum_{i=1}^{m}\xo{\lambda}_i} , 
\frac{2(1-\xo{b})}{\sum_{i=1}^{m}\xo{\lambda}_i},\frac{2(-1-\xo{b})}{\sum_{i=1}^{m}\xo{\lambda}_i};
\frac{2\xo{\lambda}}{\sum_{i=1}^{m}\xo{\lambda}_i},\xo{x}\right)$
  is a saddle point of $L_2$. $\qed$

\end{proof}

\section{Game theoretic interpretation of nearest point pair problem  between two disjoint closed convex sets}\label{game-formulation}

Based on geometric properties of class separation in the dual space, a non-cooperative game formulation is given for SVM in \cite{Cou}. 
 In this section we formulate the more general problem of finding the two closest points in two closed convex set as a search of Nash equilibrium for a two-player game.  In this game, each player chooses one point from its set  and gets a payoff given by the distance between its associated set and an hyperplane defined through the duality mapping that is located at the middle of the segment joining the points chosen by the two players.
 As a by-product this will give a game theoretic interpretation for the robust SVM too. One may find an interest in such formulation in applications where data privacy is crucial. Indeed, as each player only has knowledge of its own data points, separation can be carried out in a distributed manner where data privacy is preserved.

Given two sets $A,B \subset X$ we denote the distance between  $A$ and $B$ by $\dist(A,B)=\inf_{x \in A, y\in B} \|x-y\|$. When $A=\{x\}$, we use the simplified notation $\dist(x,B)$. The distance from a point to an hyperplane is given by \cite[Lemma 1]{DerLee}
\begin{equation}\label{distToHyp}
\dist(x_0,\left\{x\in \X: \langle x,w\rangle -c=0 \right\}) = \frac{|\langle x_0,w\rangle -c|}{\|w\|}.
\end{equation}  
Moreover, the set of nearest points in $A$ to $x\in X\setminus A$  is denoted $\prj_A(x)=\mathrm{argmin}_{y\in A}\|x-y\|$.
Suppose that $X$ is smooth which imply  that the duality mapping $M$ is single-valued \cite[Corollary 4.5]{Cio}.

Consider the following two players game. The player $i$ picks a point on the convex $C_i$, $i\in \{1,2\}$. Then the unique point $w$  of $M(x_1-x_2)$ is used to define the hyperplane $H(x_1,x_2):=\{x\in \X : \langle w,x\rangle= \langle w, \frac{x_1+x_2}{2}\rangle\}.$ This hyperplane is halfway between $x_1$ and $x_2$. Indeed
\begin{eqnarray*}
\dist(x_1,H(x_1,x_2))&=&\frac{ \left|\langle w,x_1\rangle-\langle w, \frac{x_1+x_2}{2}\rangle\right|}{\|w\|}\\
&=& \frac{ \left|\langle w, x_1-x_2 \rangle\right|}{2\|w\|}\\
&=&  \frac{ \left\|x_1-x_2 \right\|}{2}.
\end{eqnarray*}
Similarly, $d(x_2,H(x_1,x_2))=\frac{ \left\|x_1-x_2 \right\|}{2 }$.
The payoff is defined by
$$
v_i(x_1,x_2):= \dist(H(x_1,x_2),C_i), \qquad i\in\{1,2\}.
$$
If $X$ is a Hilbert space then the hyperplane is defined by
$$
H(x_1,x_2)=\{x\in \X : \langle x_1-x_2,x-\frac{x_1+x_2}{2}\rangle=0\}.
$$

This game, denoted by $G$, can be interpreted as if each player was trying to "push" the hyperplane further to himself. The payoff function $v_i$  measures how far the hyperplane is to the player.

A point $(\xo{x}_1,\xo{x}_2)$ is called a Nash Equilibrium (NE) for this game iff 
$$
\xo{x}_1\in \argmax_{x_1 \in C_1}v_1(x_1,\xo{x}_2)\quad \mbox{ and }\quad
\xo{x}_1\in \argmax_{x_2 \in C_2}v_1(\xo{x}_1, x_2).
$$
We state the  main result of this section.
\begin{theorem}\label{thm-game}
Let $C_1$ and $C_2$ be two closed convex sets in a reflexive and smooth Banach space $X$. If $C_1\cap C_2=\emptyset,$ then
 $(\xo{x}_1,\xo{x}_2)$  is (NE) for $G$ iff $\|\xo{x}_1-\xo{x}_2\|=\dist(C_1,C_2).$
 Moreover, in the that case the payoffs for both players are equal to 
$\frac{1}{2}\|x_1-x_2\|$.
\end{theorem}

The following lemmas will be used in the proof of Theorem \ref{thm-game}.
\begin{lemma}\label{lem-proj1}
\begin{enumerate}
\item
Let $w\neq 0$ be an element from $\Xs$ and $H$ the hyperplane $H=\{x:\langle x,w\rangle=c\}$. Then for each pair of points $x_1,x_2\in \X$  strictly separated by $H,$ we have
$$
\|x_1-x_2\|\geq \dist(x_1,H)+\dist(x_2,H).
$$
\item Let $C\subset \X$ be a closed convex set and $x_0\in \X.$ Suppose that there exists $x_\star \in \prj_C(x_0)$ such that  $M(x_0-x_\star)=\{w_\star\}$
then 
\begin{equation}
\langle x-x_\star,w_\star\rangle \leq 0, \qquad  \forall x\in C.
\end{equation} 

\end{enumerate}
\end{lemma}
\begin{proof}
Since  $x_1$ and $x_2$  are strictly separated, one of them is located in the positive half-space while the other one is located in the negative half-space. Suppose for example that $\langle x_1,w\rangle<c $ and $\langle x_2,w\rangle >c.$
By (\ref{distToHyp}), 
\begin{eqnarray*}
 \dist(x_1,H)+\dist(x_2,H)&=& \frac{-\langle x_1,w\rangle+ c}{\|w\|}+\frac{\langle x_2,w\rangle-c}{\|w\|}\\
 &=& \frac{-\langle x_2-x_1,w\rangle}{\|w\|}\\
 &\leq& \frac{\| x_2-x_1\|\|w\|}{\|w\|}.\\
\end{eqnarray*}
To prove the second item, let $x\in C$ and $\theta\in [0,1]$ then by the convexity of $C$, 
$\theta x+(1-\theta) x_\star \in C$. We have
\begin{eqnarray*}
0&\geq& \frac{1}{2}\|x_0-x_\star \|^2- \frac{1}{2}\|x_0-\left(\theta x+(1-\theta) x_\star)\right)\|^2\\
&=&  \frac{1}{2}\|x_0-x_\star \|^2- \frac{1}{2}\|x_0-x_\star -\theta (x- x_\star)\|^2\\
&\geq& \langle \theta (x- x_\star), w_\theta\rangle
\end{eqnarray*}
where $w_\theta \in M(x_0-x_\star -\theta (x- x_\star))$. By dividing by $\theta$ and letting $\theta$ to 0, we obtain the desired inequality since the duality mapping $M$ is norm to weak* \usc \cite[Theorem 4.12]{Cio}. $\qed$
\end{proof}

\begin{lemma}\label{lem-proj2}
Let $C_1$, $C_2$ two closed convex sets of $X$.
\begin{enumerate}
\item Let $x_1\in C_1$ and $x_2\in C_2$ such that $M(x_2-x_1)=\{w_\star \},$  then
$$
\|x_1-x_2\| =\dist(C_1,C_2)\Longleftrightarrow x_1 \in \prj_{C_1}(x_2) \mbox{ and }  x_2 \in \prj_{C_2}(x_1).
$$
Moreover, in that case 
$$
\dist(H,C_1)=\dist(H,C_2)=\frac{1}{2}\|x_1-x_2\|,
$$
where  $H$ is the hyperplane defined by $\{x\in \X : \langle w_\star,x\rangle= \langle w_\star, \frac{x_1+x_2}{2}\rangle\}.$ 
\end{enumerate}
\end{lemma}
\begin{proof}
The direct sense is obvious. Consider the reverse one. Since   $M(x_1-x_2)$ is reduced to the single element $-w_\star$, by Lemma \ref{lem-proj1} we have
\begin{eqnarray}
\langle y_1-x_1, w_\star\rangle &\leq& 0, \qquad  \forall y_1\in C_1\label{eqlem1-1},\\
\langle y_2-x_2,-w_\star\rangle & \leq& 0, \qquad  \forall y_2\in C_2\label{eqlem1-2}.
\end{eqnarray}

By Lemma \ref{lem-proj1} again, for all $y_1\in C_1$, $y_2\in C_2$ we have
\begin{equation}\label{eqlem1-3}
\|y_1-y_2\|\geq \dist(y_1,H)+\dist(y_2,H).
\end{equation}
Moreover, we have
\begin{eqnarray}
\dist(y_1,H)- \frac{1}{2}\|x_1-x_2\|
&=& \frac{\langle y_1- \frac{x_1+x_2}{2},  w_\star\rangle}{\|w_\star\|}
 -\frac{\langle x_1-x_2,w_\star \rangle }{2\|w_\star\|}\ \nonumber\\
 &\geq & \frac{\langle y_1-x_1,w_\star\rangle}{\|w_\star\|}\nonumber\\
 &\geq & 0 \qquad (\mbox{by (\ref{eqlem1-1})}) \label{eqlem1-4}
\end{eqnarray}
and in the same manner we obtain
\begin{eqnarray}
\dist(y_2,H)- \frac{1}{2}\|x_1-x_2\| \geq  0. \label{eqlem1-5}
\end{eqnarray}
Summing (\ref{eqlem1-3}),(\ref{eqlem1-4}) and (\ref{eqlem1-5}) we get $ \|y_1-y_2\|\geq \|x_1-x_2\|$ for all $y_1\in C_1, y_2\in C_2$  which proves that $\dist(C_1,C_2)=\|x_1-x_2\|$.\\
By (\ref{eqlem1-4}) (respectively (\ref{eqlem1-5})), we have $
\dist(H,C_1)\geq \frac{1}{2}\|x_1-x_2\|
$   (respectively $\dist(H,C_2)\geq \frac{1}{2}\|x_1-x_2\|$)
and the equality is achieved by $x_1$ (respectively $x_2$) since $\dist(H,x_1)= \frac{1}{2}\|x_1-x_2\|$ (respectively $\dist(H,x_2)=\frac{1}{2}\|x_1-x_2\|$). $\qed$
\end{proof}

\begin{proof} (of Theorem \ref{thm-game})
Suppose that $(\xo{x}_1,\xo{x}_2)$ is a (NE). Let $x_1 \in P_{C_1}(\xo{x}_2)$ and $x_2 \in P_{C_2}(\xo{x}_1)$, their existence is ensured by the reflexivity of $X$  \cite{BorFit}. By Lemma \ref{lem-proj1} we have
\begin{eqnarray}
\|\xo{x}_2-x_1\|&\geq& \dist(\xo{x}_2,H(\xo{x}_1,\xo{x}_2))+ \dist(x_1,H(\xo{x}_1,\xo{x}_2))\nonumber\\
&\geq& \frac{1}{2}\|\xo{x}_2-\xo{x}_1\|+ \dist(C_1,H(\xo{x}_1,\xo{x}_2))\nonumber\\
&\geq&  \frac{1}{2}\|\xo{x}_2-\xo{x}_1\|+ \dist(C_1,H(x_1,\xo{x}_2))\label{eqThG-1} \\
&\geq&  \frac{1}{2}\|\xo{x}_2-\xo{x}_1\|+ \frac{1}{2}  \|\xo{x}_2-x_1\|. \label{eqThG-2}
\end{eqnarray}
The inequality (\ref{eqThG-1}) comes from the fact that $(\xo{x}_1,\xo{x}_2)$ is a (NE)
while (\ref{eqThG-2}) from Lemma \ref{lem-proj1}
applied with two convex sets $C_1$ and $\{\xo{x}_2\}$. We obtain from (\ref{eqThG-2}) $\|\xo{x}_2-\xo{x}_1\|\leq \|\xo{x}_2-x_1\|$ which means that $\xo{x}_1\in \prj_{C_1}(\xo{x}_2).$ Proceeding by the same way we obtain $\xo{x}_2\in \prj_{C_2}(\xo{x}_1)$, that is by Lemma \ref{lem-proj2} $\|\xo{x}_1-\xo{x}_2\|=\dist(C_1,C_2)$.
Conversely, let$(\xo{x}_1,\xo{x}_2)$ such that $\|\xo{x}_1-\xo{x}_2\|=\dist(C_1,C_2)$ and suppose by contradiction that  $(\xo{x}_1,\xo{x}_2)$ is not (NE). Then there exist $x_2\in C_2$ (or $x_1\in C_1$) such that
\begin{equation}\label{eqThG-3} 
\dist(H(\xo{x}_1,x_2),C_2) > \dist(H(\xo{x}_1,\xo{x}_2),C_2)= \frac{1}{2}\|\xo{x}_1-\xo{x}_2\|.
\end{equation}
By Lemma \ref{lem-proj2} we have
\begin{eqnarray*}
\|\xo{x}_1-\xo{x}_2\|&\geq& \dist(\xo{x}_1,H(\xo{x}_1,x_2)) + \dist(\xo{x}_2,H(\xo{x}_1,x_2))\\
&=& \frac{1}{2}\|\xo{x}_1-x_2\|  +\dist(\xo{x}_2,H(\xo{x}_1,x_2))\\
&\geq& \frac{1}{2}\|\xo{x}_1-x_2\|  +\dist(C_2,H(\xo{x}_1,x_2))\\
&\geq& \frac{1}{2}\|\xo{x}_1-x_2\|  +\frac{1}{2}\|\xo{x}_1-\xo{x}_2\| \qquad (\mbox{by } (\ref{eqThG-3})).
\end{eqnarray*}
that is, $\|\xo{x}_1-\xo{x}_2\|>\|\xo{x}_1-x_2\|$. This contradicts the fact that $\xo{x}_1$ and $\xo{x}_2$ are the nearest neighbours. $\qed$

\end{proof}


\section{The non separable case}

Let us now suppose that $K_-$ and $K_+$ are non-linearly separable. A linear robust soft margin SVM training can be formulated by using slack variables  which  measure  the  degree  of  misclassification    of the    observations
 leading to the following relaxed version 
$$
\rsvmc{C}\qquad 
\begin{array}{ll}
 \displaystyle\min_{(w,b,\xi)\in \Xs\times\R\times \R^m_+} & \frac{1}{2}\nd{w}^2 +C\sum_{i=1}^m \xi_i \\
 \sth  & \displaystyle\min_{x_i\in K_i}y_i(\langle x_i, w \rangle + b)\geq 1-\xi_i,\quad \dea{i}{m},
\end{array} 
$$
where $C>0$ is a problem specific constant controlling the trade-off between margin (generalisation)
and classification. The optimistic counterpart of its corresponding uncertain dual is
$$
\odrsvmc{C}\qquad 
\begin{array}{ll}
 \displaystyle\sup_{(\lambda,x)\in \R_+^m\times K} &\sum_{i=1}^m \lambda_i - \frac{1}{2}\left\| \sum_{i=1}^m \lambda_i y_i x_i \right\|^2 \\
 \sth  & \sum_{i=1}^m \lambda_i y_i =0,\\
      & \lambda_i \leq C,\ \dea{i}{m}.  
\end{array} 
$$
In a similar way we formulate the relaxed version of $\rcm$
$$
\rcmd{D} \qquad 
\begin{array}{ll}
 \displaystyle\min_{(w,\alpha,\beta,\xi)\in \Xs\times\R\times\R\times \R^m_+} & \frac{1}{2}\nd{w}^2  -(\alpha-\beta)+D\sum_{i=1}^m\xi_i\\
 \sth  & \displaystyle\min_{x_i\in K_i} \langle x_i, w \rangle \geq \alpha -\xi_i,\quad  i \in I_+,\\
      & \displaystyle\max_{x_i\in K_i} \langle x_i, w \rangle \leq \beta +\xi_i,\quad  i \in I_-\\
\end{array}
$$
whose optimistic counterpart of its corresponding uncertain dual is
$$
\odrcmd{D}\qquad 
\begin{array}{ll}
 \displaystyle\sup_{(\lambda,x)\in \R_+^m\times K} &  - \frac{1}{2}\left\| \sum_{i=1}^m \lambda_i y_i x_i \right\|^2 \\
 \sth  & \sum_{i\in I_\circ} \lambda_i =1, \quad \circ \in \{+,-\},\\
      & \lambda_i \leq D,\ \dea{i}{m}.
\end{array} 
$$
This is in fact not else but the problem of minimizing the (squared) distance
between the two convex sets
$$K_\circ(D)=
\left\{\displaystyle\sum_{i\in {I_\circ}}\alpha_i x_i : x_i \in K_i, 0\leq \alpha_i\leq D, i \in I_\circ \mbox{ and }   \displaystyle\sum_{i \in I_\circ}\alpha_i =1 \right\}, \circ\in\{-,+\},
$$
corresponding to Reduced Convex Hull following the terminology of \cite{BenBre}.  Under this form, it is clear that reducing $D$ sufficiently will ensure separability of the problem.
The results established in the separable case can, by almost similar arguments, be extended to this non-separable case.
In particular, as in \cite{BenBre}, we can show that optimizing R-SVM($C$) is equivalent to optimizing R-CM($D$). The parameters $C$ and $D$ are related by multiplication of a constant factor as shown by the following theorem.
\begin{theorem}
Assume that the uncertainty sets $K_i$,  \dea{i}{m}, are convex and weakly compact, then
\begin{enumerate}
\item 
If $(\xo{w},\xo{b},\xo{\xi};\xo{\lambda},\xo{x})$  is solution for $\rsvmc{C}-\odrsvmc{C}$, with $\xo{w}\neq 0$, then
$$\left(\frac{2\xo{w}}{\sum_{i=1}^{m}\xo{\lambda}_i} , 
\frac{2(1-\xo{b})}{\sum_{i=1}^{m}\xo{\lambda}_i},\frac{2(-1-\xo{b})}{\sum_{i=1}^{m}\xo{\lambda}_i}, \frac{2\xo{\xi}}{\sum_{i=1}^{m}\xo{\lambda}_i};
\frac{2\xo{\lambda}}{\sum_{i=1}^{m}\xo{\lambda}_i},\xo{x}\right)$$ is a solution to $\rcmd{\frac{2C}{\sum_{i=1}^{m}\xo{\lambda}_i}}-
\odrcmd{\frac{2C}{\sum_{i=1}^{m}\xo{\lambda}_i}}.$
\item
If $(\xo{w},\xo{\alpha},\xo{\beta},\xo{\xi};\xo{\lambda},\xo{x})$ is solution to $\rcmd{D}-\odrcmd{D}$, with $\xo{w}\neq 0$, then  
$$\left(\frac{2\xo{w}}{\xo{\alpha}-\xo{\beta}} ,
-\frac{\xo{\alpha}+\xo{\beta}}{\xo{\alpha}-\xo{\beta}},\frac{2\xo{\xi}}{\xo{\alpha}-\xo{\beta}} ;  \frac{2\xo{\lambda}}{\xo{\alpha}-\xo{\beta}},\xo{x}\right)$$
is a solution for $\rsvmc{\frac{2D}{\xo{\alpha}-\xo{\beta}}}-\odrsvmc{\frac{2D}{\xo{\alpha}-\xo{\beta}}}$.

\end{enumerate}
\end{theorem}

\section{Conclusion}
This theoretical analysis is an additional step towards the generalisation of formulations of binary classification problems in Banach spaces. In \cite{DerLee}, it had already been shown that classical SVM formulations nicely extends to Banach spaces by the use of semi-inner products. The authors had shown that most of hard margin separation results in Hilbert spaces remain valid in the non Euclidean setting when considering an appropriate alternative to inner products. In our study, we show that using the duality product, we not only also retrieve the binary classification formulation but robust formulations can also be derived when data uncertainties lie in Banach spaces. Furthermore, using the classification formulation based on the duality product, we show that game theoretic interpretations can also be made. This bridge between game theory and classification of complex data (represented in Banach spaces rather than Hilbert spaces) opens new opportunities for exploiting theoretical and numerical results from both worlds.

\section*{Acknowledgment} The authors would like to thank Prof. Jean-Baptiste Hiriart-Urruty (Universit\'{e} Toulouse III- Paul Sabatier, France) for his  insightful and constructive discussions about this research. This work has partially benefited from the AI Interdisciplinary Institute ANITI. ANITI is funded by the French ”Investing for the Future - PIA3” program under
the Grant agreement \# ANR-19-PI3A-0004.

%
%



\end{document}